\documentclass[11pt]{article}

\usepackage[preprint]{acl}
\definecolor{darkmagenta}{HTML}{A0007A} 
\hypersetup{citecolor=darkmagenta, linkcolor=darkmagenta, urlcolor=darkmagenta}

\usepackage{amsfonts}
\usepackage{amsmath}
\usepackage{amssymb}
\usepackage{amsthm}
\usepackage{times}
\usepackage{latexsym}
\usepackage[inline]{enumitem}

\usepackage[T1]{fontenc}

\usepackage[utf8]{inputenc}

\usepackage{microtype}

\usepackage{inconsolata}

\usepackage{graphicx}

\usepackage{todonotes}
\makeatletter
\newcommand*\iftodonotes{\if@todonotes@disabled\expandafter\@secondoftwo\else\expandafter\@firstoftwo\fi}  %
\makeatother

\usepackage{ellipsis}
\makeatletter
\renewcommand*{\mathellipsis}{%
  \mathinner{%
    \kern\ellipsisbeforegap%
    {\ldotp}\kern\ellipsisgap%
    {\ldotp}\kern\ellipsisgap%
    {\ldotp}\kern\ellipsisaftergap%
  }%
}
\renewcommand*{\dotsb@}{%
  \mathinner{%
    \kern\ellipsisbeforegap%
    {\cdotp}\kern\ellipsisgap%
    {\cdotp}\kern\ellipsisgap%
    {\cdotp}\kern\ellipsisaftergap%
  }%
}
\renewcommand*{\@cdots}{%
  \mathinner{%
    \kern\ellipsisbeforegap%
    {\cdotp}\kern\ellipsisgap%
    {\cdotp}\kern\ellipsisgap%
    {\cdotp}\kern\ellipsisaftergap%
  }%
}
\renewcommand*{\ellipsis@default}{%
  \ellipsis@before
  \kern\ellipsisbeforegap
  .\kern\ellipsisgap
  .\kern\ellipsisgap
  .\kern\ellipsisgap
  \ellipsis@after\relax}
\renewcommand*{\ellipsis@centered}{%
  \ellipsis@before
  \kern\ellipsisbeforegap
  .\kern\ellipsisgap
  .\kern\ellipsisgap
  .\kern\ellipsisaftergap
  \ellipsis@after\relax}
\AtBeginDocument{%
  \DeclareRobustCommand*{\dots}{%
    \ifmmode\@xp\mdots@\else\@xp\textellipsis\fi}}
\def\ellipsisgap{-.05em}
\def\ellipsisbeforegap{-.05em}
\def\ellipsisaftergap{.05em}
\makeatother

\newtheorem{theorem}{Theorem}

\newtheorem{definition}{Definition}

\usepackage{thmtools}
\usepackage{thm-restate}
\usepackage{etoc}
\usepackage{cleveref}
\crefname{proposition}{Proposition}{}
\Crefname{proposition}{Proposition}{}
\crefname{section}{\S}{\S\S}
\Crefname{section}{\S}{\S\S}
\crefname{table}{Table}{}
\crefname{figure}{Figure}{}
\crefname{algorithm}{Algorithm}{}
\crefname{equation}{Eq.}{}
\Crefname{equation}{Eq.}{}
\crefname{appendix}{App.}{}
\crefformat{section}{\S#2#1#3}

\newcommand{\defn}[1]{\textbf{#1}}
\newcommand{\emdash}{\nobreak\hspace{0pt}---\hspace{0pt}} 
\renewcommand{\citeposs}[1]{\citeauthor{#1}'s (\citeyear{#1})}
\newcommand{\defeq}{\mathrel{\stackrel{\textnormal{\tiny def}}{=}}}

\newcommand{\alphabet}{{\Sigma}}
\newcommand{\eosalphabet}{{\alphabet_{\textnormal{\textsc{e}}}}}
\newcommand{\kleene}[1]{{#1^*}}

\newcommand{\strlen}{{T}}
\newcommand{\eos}{{\textnormal{\textsc{eos}}}}

\newcommand{\pLM}{{q}}
\newcommand{\pGen}{{p}}
\newcommand{\qLM}{{q}}
\newcommand{\entropy}{{\mathrm{H}}}

\newcommand{\KL}[2]{D_{\mathrm{KL}}\!\left(#1\,\|\,#2\right)}
\newcommand{\unit}{{u}}
\newcommand{\units}{{\boldsymbol{u}}}
\newcommand{\difficulty}{{\mathrm{d}}}
\newcommand{\pHuman}{{p_{\mathrm{H}}}}
\newcommand{\family}{{\mathcal{Q}}}

\newcommand{\pData}{{p_{\mathrm{C}}}}
\newcommand{\memory}{{m}}
\newcommand{\metric}{{M}}
\newcommand{\reals}{{\mathbb{R}}}
\newcommand{\scoring}{{g}}
\newcommand{\prefix}[1]{\vec{#1}} 
\newcommand{\qdiff}{{\pLM_{\difficulty}}}
\newcommand{\qdiffprefix}{{\prefix{\pLM}_{\difficulty}}} 
\newcommand{\pDataprefix}{{\prefix{\pGen}_{\mathrm{C}}}} 
\newcommand{\pHumanprefix}{{\prefix{\pGen}_{\mathrm{H}}}} 
\DeclareMathOperator*{\E}{\mathbb{E}}
\usepackage{bbm}
\newcommand{\ind}[1]{\mathbbm{1}\{#1\}} 

\title{Surprisal Theory is Tautological 
(without Rational Grounding)}

\author{Ryan Cotterell \\
  ETH Z{\"u}rich \\
  \href{mailto:ryan.cotterell@inf.ethz.ch}{\texttt{ryan.cotterell@inf.ethz.ch}}}

\begin{document}
\maketitle
\makeatletter
\raggedbottom
\let\tautology@raggedtextbottom\@textbottom
\def\@textbottom{\ifnum\c@page=\@ne \tautology@raggedtextbottom \fi}
\makeatother
\etocdepthtag.toc{mainmatter} 
\begin{abstract}
Surprisal theory holds that the human processing difficulty of a linguistic unit is an affine function of its surprisal in context under some language model.
I argue this claim is tautological, i.e., for any non-negative difficulty measure over units in context, there exists a language model whose surprisal is an affine function of it under mild technical conditions.
Therefore, because any pattern of difficulty is consistent with \emph{some} language model, without an additional constraint on the language model itself, surprisal theory makes no falsifiable predictions.
I further claim that the tautology was obscured by an assumption implicit in two decades of psycholinguistic work---that the language model relevant for studying human sentence processing is the distribution that generated the training corpus.
Recent empirical work, however, has undermined this assumption, demonstrating that better corpus models can be \emph{worse} predictors of processing difficulty.
I conclude that breaking the tautology requires a rationalist intervention, i.e., the language models used for the study of human sentence processing must be derived from an independently motivated model of the comprehender, which could be based on, for instance, memory constraints or processing goals, and that, crucially, does not depend on the behavioral data surprisal theory is meant to explain.
\end{abstract}

\section{Introduction}\label{sec:intro}

A foundational observation in psycholinguistics is that human language comprehension is \emph{incremental}.
In particular, comprehenders do not wait until the end of an utterance before beginning to interpret it.
Instead, they build a partial interpretation of that utterance unit by unit on the fly \citep{altmann1999,tanenhaus1995}.\footnote{Language processing is not fully incremental.
Eye-tracking data contain regressive saccades in which readers often revisit earlier material \citep{wilcox2024regressions}.
Nevertheless, in English and other languages with similar orthography, humans' eye movements tend to proceed left to right, and, correspondingly, the standard modeling assumption in the surprisal literature is that difficulty at position $t$ depends on $\units_{<t}$ alone.
I adopt this assumption throughout; extending the framework to account for regressions is an open problem.}
If comprehension proceeds incrementally, then the total effort of understanding an utterance $\units = \unit_1 \cdots \unit_{\strlen}$ naturally decomposes into a series of per-step costs, one for each unit integrated into the evolving representation.
Incrementality thus motivates the central theoretical primitive of this paper---\defn{processing difficulty} $\difficulty(\unit_t \mid \units_{<t}) \geq 0$, a non-negative function that quantifies the effort a comprehender expends to integrate unit $\unit_t$ given the preceding context $\units_{<t} \defeq \unit_1 \cdots \unit_{t-1}$;\footnote{The construct $\difficulty$ is agnostic about the source of the processing effort: it subsumes prediction disconfirmation, the cost of integrating $\unit_t$ into the evolving representation, and any reanalysis of earlier material that $\unit_t$ triggers.
Because reanalysis at position $t$ is initiated by observing $\unit_t$ in context $\units_{<t}$, its cost remains a function of the pair $(\unit_t, \units_{<t})$ and is therefore representable by $\difficulty$; indeed, under \citeposs{levy2008} resource-reallocation interpretation, reanalysis costs are part of what surprisal is meant to quantify.} the notation is formally introduced in \cref{sec:preliminaries}.\footnote{I treat $\difficulty$ as a function of the linguistic unit and its context alone, abstracting over reader identity, reader state, and presentation modality.
This follows the standard linear mixed-effects regression framework used in psycholinguistics \citep{smith2013,goodkind2018}, which decomposes observed reading times into a surprisal-driven component shared across readers, control predictors (word length, frequency, position), and residual noise.
If $\difficulty$ were instead reader-specific, the tautology construction in \Cref{eq:tautology-construction} would produce a different language model for each reader---but surprisal theory posits a single language model shared across readers, which is itself a testable constraint that the present argument does not exploit.}
In practice, processing difficulty is operationalized through a variety of behavioral and neural measures---self-paced reading times, eye-tracking fixation durations, and event-related potentials, \textit{inter alia}---and a successful theory of sentence processing must explain its variation across positions and contexts.
However, for the exposition of this paper, I abstract away from specific operationalizations with the construct $\difficulty$.

\subsection{The Tautology of Surprisal Theory}
Surprisal theory \citep{hale2001,levy2008} is a widely adopted information-theoretic framework for prediction-based processing \citep{kuperberg2016}.
It claims that the processing difficulty of a unit in context is an affine function of that unit's contextual surprisal under some language model $\pLM$ \citep{shain2024};\footnote{\citet{shain2024} provide large-scale evidence that the relationship between word predictability and reading time is linear in log-probability---i.e., that the affine form in \Cref{eq:surprisal-intro} is a better fit than alternatives such as a logarithmic or power-law relationship in raw probability.
The affine formulation is important for my tautology argument: the construction in \Cref{eq:tautology-construction} yields $a = 1$ and a context-dependent intercept $b = -\log Z(\units_{<t})$, which is exactly the affine form.
A strictly proportional formulation ($b = 0$) would leave a gap, since the partition function varies with context.} in symbols, surprisal theory corresponds to
\begin{equation}\label{eq:surprisal-intro}
    \difficulty(\unit_t \mid \units_{<t}) = -a\log \prefix{\pLM}(\unit_t \mid \units_{<t}) + b(\units_{<t}),
\end{equation}
where $a > 0$ is a constant and $b \colon \kleene{\alphabet} \to \reals$ is a context-dependent baseline.
This paper argues that, in its current conceptualization, surprisal theory is not falsifiable \citep{popper1959} without constraints that go beyond current practice and, therefore, tautological.
Because the theory leaves $\pLM$ as a free parameter, \emph{any} observed pattern of processing difficulty can be accommodated.

\paragraph{A Formal Tautology.}
To see the emergence of a formal tautology, for any difficulty measure $\difficulty$, one can simply construct a language model whose surprisal is an affine function of it:
\begin{equation}\label{eq:tautology-construction}
\!\qdiffprefix(\unit_t \mid \units_{<t}) \defeq \frac{\exp({-\difficulty(\unit_t \mid \units_{<t})})}{\sum_{\unit' \in \eosalphabet} \exp({-\difficulty(\unit' \mid \units_{<t})})}.
\end{equation}
Under this construction, $-\log \qdiffprefix(\unit_t \mid \units_{<t})= \difficulty(\unit_t \mid \units_{<t}) + \log Z(\units_{<t})$, where $Z(\units_{<t})$ is the partition function---so \Cref{eq:surprisal-intro} holds with $a = 1$ and $b(\units_{<t}) = -\log Z(\units_{<t})$.\footnote{In psycholinguistic practice, context-dependent baselines are absorbed into control predictors in the regression model.
The affine formulation with context-dependent $b$ is therefore the natural statement of the theory.}
Under mild conditions, explicated in \cref{sec:tautology-formal}, this construction yields a valid language model.\footnote{The tautology argument applies more broadly to any theory that defines processing difficulty purely as a function of the language model's probability distribution.}
The consequence is that, without further constraining $\pLM$, surprisal theory's claim that difficulty scales with surprisal is vacuous, i.e., it is compatible with every possible choice of $\difficulty$.

\paragraph{The Grounding Problem.}
Escaping the tautology thus requires two choices that are external to surprisal theory.
First, one must fix the family of language models $\family$ under consideration, which encodes a structural commitment---for example, one could restrict $\family$ to the set of all probabilistic context-free grammars \citep[PCFGs;][]{booth1969}.
Second, one must give a criterion by which the language model $\pLM \in \family$ is selected, e.g., whether it is fit to a corpus, derived from cognitive principles, or optimized against behavioral data.
I call the need to make both choices surprisal theory's \defn{grounding problem}.
\citet{hale2001} confronted the first choice directly.
Invoking his Principle~1, which he terms strong competence \citep{chomsky1965}, he restricts $\family$ to the set of PCFGs on rationalist grounds---specifically, he contends that ``the human sentence processing mechanism directly uses rules of grammar in its operation.''
I examine the second choice further in \cref{sec:intro-propensity}.

\paragraph{The Drift.}
One possible rebuttal of my claimed tautology is that proponents of surprisal theory have never actually claimed that just \emph{any} language model will suffice.
On this point, I agree; this is not what surprisal theory's proponents are \emph{claiming}.
Instead, my critique focuses on the \emph{practice} exhibited in many recent papers on surprisal theory.
The problem, as I see it, is that subsequent empirical work has progressively \emph{drifted} from commitments such as \citeposs{hale2001}.
For instance, practitioners have moved from using PCFGs to estimate surprisal to using recurrent neural networks \citep{hochreiter1997} and transformers \citep{vaswani2017}---whichever architecture is currently in vogue---and most papers on the subject adopt model families whose cognitive relevance the authors do not even argue for.\looseness=-1

\paragraph{The Cognitivist Assumption.}
A second possible rebuttal appeals to \citeposs{marr1982} levels of analysis: the \emph{computational} level specifies what problem a system solves and why, the \emph{algorithmic} level specifies the representations and procedures by which it solves it, and the \emph{implementational} level specifies how those procedures are physically realized.
Surprisal theory was, from its inception, presented as a computational-level theory, and one could argue that the cognitive grounding of $\family$ is an algorithmic-level detail---so that the drift just described is irrelevant to a claim that was never algorithmic in the first place.\footnote{\citet{buxolugo2026} adduce a similar argument about surprisal theory in terms of \citeposs{marr1982} levels.
Like me, they consider surprisal theory to be a computational-level account.
However, they then argue that the specific algorithmic-level commitments, e.g., the choice of language model used to compute surprisal, sneak into the analysis, resulting in a confound.
This argument has some substance, but I worry that their claim proves too much.
Indeed, \emph{any} measurement of a computational theory would fall victim to their concern.
In fairness, the authors pre-empt this objection (p.~25), arguing that the problem lies specifically in using ``black box'' models as measurement tools to test a computational theory.
On my reading, \citet{buxolugo2026} fail to offer a definition of blackboxhood, or even a criterion by which one could claim that a transformer language model is blackbox, but a PCFG is not.
Moreover, I remain unconvinced that an ontological distinction between blackbox and non-blackbox models can be pinned down.
My argument, in contrast, concerns a tautology present purely at the computational level of analysis.
}
Even granting this, one must still say what $\pLM$ is meant to approximate.
My answer is a theoretical construct $\pHuman$, the comprehender's internal language model; I call the assumption that $\pLM$ approximates $\pHuman$ the \defn{cognitivist assumption}.
\citet{levy2008} comes closest to making the cognitivist assumption explicit, writing ``the cognitive entity of primary interest is the resulting probabilistic [unit] model alone'' (p.~1138), and that the choice of grammar ``is not a commitment to any particular grammatical formalism as the backbone of sentence comprehension'' but rather ``a formal means of estimating what expectations about upcoming [units] in a sentence implicitly arise'' (p.~1138).
Under the cognitivist assumption, surprisal theory reduces to the claim that\looseness=-1
\begin{equation}\label{eq:surprisal-cognitivist}
\!\difficulty(\unit_t \mid \units_{<t}) = -a\log \pHumanprefix(\unit_t \mid \units_{<t}) + b(\units_{<t}).
\end{equation}
But the cognitivist assumption does not, by itself, solve the grounding problem: without an independent characterization of $\pHuman$ that is \emph{not} grounded in the behavioral data that surprisal theory seeks to explain, any pattern of difficulty can be explained by positing a $\pHuman$ whose surprisal matches it.
It merely relocates the problem; it does not remove it.

\subsection{An Analogy from Evolutionary Biology}
\label{sec:intro-tautology}
Tautologies are not new in science.
An instructive parallel comes from evolutionary biology.
\citeposs{darwin1859} principle of natural selection---that heritable traits conferring reproductive advantage spread through a population---is distilled in the slogan \emph{survival of the fittest}.
But if fitness simply means reproductive success, the slogan reduces to ``those who reproduce most, reproduce most''---a tautology.
The observation has a long history \citep{butler1882,sinnott1958,waddington1960} and was given its most sustained treatment by \citet{peters1976}; Popper famously concluded that natural selection was a ``metaphysical research programme'' rather than a testable theory \citep{popper1974}.
The criticism was not that evolution was false, but that its central explanatory construct was defined in terms of the very phenomenon it was meant to explain.\looseness=-1

The resolution came through what \citet{mills1979} called the
propensity interpretation of fitness; see also \citet{brandon1978}.
On this view, fitness is not a measure of actual reproductive output but a \defn{propensity}---a causal tendency, grounded in an organism's morphology, physiology, and behavior in a given environment, to survive.
Under the propensity interpretation, natural selection makes a falsifiable claim: organisms with higher propensities reproduce more than those with lower ones.
The key structural feature of the propensity interpretation---and the one I will exploit below---is that it grounds an explanatory construct in physical or structural properties of the organism itself and how it relates to its environment.
Crucially, this grounding is \emph{independent} of the outcome the construct is meant to explain---reproductive success.
A propensity interpretation of surprisal theory would require an analogous move: grounding the choice of $\pLM$ in independently observable properties of the comprehender---working memory capacity, parsing strategy, attention allocation policy---rather than in the reading-time data the theory is meant to predict.
I discuss concrete proposals along these lines in \cref{sec:rationalist}.\looseness=-1

\subsection{How Surprisal Theory Lost Its Propensity}
\label{sec:intro-propensity}
While rarely stated explicitly, I contend that surprisal theory \emph{did} historically have a propensity interpretation, which I explicate below.
Let $\pData$ denote the data-generating language model of the corpus---the distribution over strings that produced the texts used to estimate $\pLM$.
Just as fitness under the propensity interpretation is grounded in how the organism interacts with its environment---observable physical properties \emph{independent} of reproductive outcomes---$\pData$ is grounded in how humans produce text in the wild, observable independently of any processing difficulty measurements.\footnote{One might object that the independence is not absolute, i.e., the same cognitive system that produces text also processes it, so $\pData$ and processing difficulty are not causally unrelated.
The dependence may run deeper than mere causal correlation: influential theories of language processing posit shared representations for production and comprehension \citep{pickering2004,dell1986}, in which case $\pData$ is not independent of $\pHuman$ but rather a noisy reflection of it.
If production and comprehension share the same underlying language model, the corpus assumption and the cognitivist assumption would collapse into one---an identification that, if correct, would simplify the grounding problem considerably.
The point, for present purposes, is that $\pData$ can be \emph{estimated} without recourse to reading-time data, which is sufficient for the propensity move to go through.
I concede, however, that the relationship between production and comprehension distributions deserves more careful treatment than I am able to give it here.\looseness=-1
}

To state it more concisely, I call the assumption that I contend has anchored two decades of empirical work the \defn{corpus assumption}.
Formally, the corpus assumption holds that $\pLM \in \family$ is chosen such that $\pLM = \pData$.
Or, when equality is not achievable, one should choose $\pLM$ so as to minimize $\KL{\pData}{\pLM}$.
Under the corpus assumption, surprisal theory makes the following claim:
\begin{equation}\label{eq:corpus-surprisal}
\!\difficulty(\unit_t \mid \units_{<t}) = -a\log \pDataprefix(\unit_t \mid \units_{<t}) + b(\units_{<t}).
\end{equation}
One wrinkle, of course, is that $\pData$ is not known.
However, it can be approximated by estimating a language model from a corpus of texts assumed to have been sampled from $\pData$.
The corpus assumption led to two concrete payoffs.
First, it made finding a $\pLM$ to experiment with in practice approachable: one can estimate a language model $\qLM$ from a corpus and treat it as an approximation to $\pData$.
Second, under maximum-likelihood estimation, the corpus assumption entails what I call the \defn{scaling implication}---as estimation improves, a model $\qLM$ should asymptotically yield surprisal values that better correlate with processing difficulty.
This implication is a direct consequence of the corpus assumption, as I formally argue in \cref{sec:tautology}, and forces surprisal theory to make a falsifiable prediction.

As a caveat, I note that, in practice, the corpus assumption has seldom been enforced in pure form.
Studies routinely compare many language models by their fit to behavioral data and report the best-fitting configurations \citep[e.g.,][]{goodkind2018,wilcox2023,oh2023}, so that subsequent work, in adopting those models, implicitly conditions the choice of $\pLM$ on the very data the theory is meant to explain; hyperparameter and checkpoint selection also introduce a logical dependence, albeit a mild one.
At the extreme, \citet{kiegeland2024} fine-tune a language model on reading times directly; I return to the status of this strategy in \cref{sec:tautology}.\footnote{In the same vein, \citet{yoshida2026} fine-tune neural language models on garden-path reading times, yielding an ``existence proof'' of a language model whose surprisal reproduces garden-path effects that off-the-shelf models notoriously underestimate---and they note that such constructibility may render surprisal theory ``unfalsifiable in practice.''
Their empirical construction is the practical counterpart of the formal one in \cref{sec:tautology-formal}; notably, however, their fine-tuned models still fail to reproduce memory-based effects such as the subject--object relative-clause asymmetry, showing that the family reachable by fine-tuning in practice is not yet universal---a contingent bound on flexibility that, as the formal argument shows, cannot be relied upon in principle.}

Initially, results suggested that contextual surprisal derived from larger language models did in fact correlate more strongly with a variety of behavioral measurements, e.g., self-paced reading times \citep{smith2013,wilcox2023}, eye-tracking measures \citep{boston2008,demberg2008}, and neural signals \citep{frank2015,giulianelli2024} across languages and model architectures.
Recently, however, a growing body of evidence appears to falsify the scaling implication---and, with it, the corpus assumption specifically.
Specifically, \citet{oh2023} reported that the correlation between surprisal and reading times improves as transformers grow larger, but only up to a point---beyond which it \emph{degrades} as models become better approximations to $\pData$.\footnote{\citet{kuribayashi2025} argue that this degradation is an artifact of evaluating only final-layer representations: middle layers of larger models retain strong alignment with human reading times. 
Relatedly, \citet{tsipidi2026} find that early-layer representations outperform surprisal itself in predicting early-pass reading measures for English across two datasets. However, their results were less consistent in other languages, and these findings do not rescue the scaling implication, which concerns the model's closeness to $\pData$.}
This finding has since been replicated and extended.
\citet{oh2024} confirmed the inverse relationship across four language model families and four reading corpora.
\citet{lin2025} demonstrate that the same inverse trend appears for fMRI-based neural measures, extending the finding beyond reading-time data.
Most definitively, \citet{oh2025leakage} claim to have ruled out data leakage as an explanation, showing that the inverse relationship persists even with models trained on data from which reading-time corpora have been removed.
Assuming these findings are correct, then the corpus assumption is no longer a viable foothold for surprisal theory.\footnote{If one considers a corpus of conversational speech rather than web text, the scaling implication might still hold for models trained on such corpora.
Distinguishing ``no data-generating distribution $\pData$ yields the right $\pLM$'' from ``the right $\pData$ has simply not been identified'' is an open empirical question that the formal argument of this paper does not resolve.\looseness=-1}

\section{Formal Preliminaries}\label{sec:preliminaries}\label{sec:background}

\paragraph{Alphabets, strings, and language models.}

An \defn{alphabet} $\alphabet$ is a finite, non-empty set whose elements I call \defn{units}---for instance, words, characters, or subword tokens.\footnote{See \citet{kiegeland2026units} for a careful treatment of the choice of unit inventory in surprisal-based analyses.}\footnote{Although $\alphabet$ is finite, it can encode structured objects with unbounded complexity through serialization into strings. \citet{snaebjarnarson2026} show how infinite-state languages can be represented over a finite alphabet via such encodings.}
The \defn{Kleene closure} $\kleene{\alphabet}$ is the set of all finite strings over $\alphabet$, including the empty string $\varepsilon$.
I denote individual units by $\unit$ (or $\unit_t$ when indexed by position), strings by $\units = \unit_1 \cdots \unit_\strlen$ where $|\units| = \strlen$ denotes the length, and prefixes of a string $\units$ are denoted by $\units_{<t} = \unit_1 \cdots \unit_{t-1}$.
I augment $\alphabet$ with a distinguished end-of-sequence symbol $\eos \notin \alphabet$ and write $\eosalphabet \defeq \alphabet \cup \{\eos\}$ for the $\eos$-augmented alphabet.
A \defn{language model} $\pGen$ is a probability distribution over $\kleene{\alphabet}$, i.e., $\sum_{\units \in \kleene{\alphabet}} \pGen(\units) = 1$.
I use $\pGen$ as a generic language model throughout this section; the specific language models $\pLM$, $\pHuman$, and $\pData$ introduced in \cref{sec:intro} are all instances of $\pGen$.
I say $\pGen$ has finite expected length if $\E_{\units \sim \pGen}[|\units|] < \infty$.
The \defn{prefix probability} of a string $\units$ under $\pGen$ is $\prefix{\pGen}(\units) \defeq \sum_{\units' \in \kleene{\alphabet}} \pGen(\units \units')$, i.e., the probability that a string sampled from $\pGen$ begins with $\units$.\looseness=-1

\paragraph{From language models to conditionals.}
For each prefix $\units_{<t} \in \kleene{\alphabet}$, the language model induces a conditional distribution over the extended alphabet $\eosalphabet$, given by $\prefix{\pGen}(\unit_t \mid \units_{<t}) = \prefix{\pGen}(\units_{<t} \unit_t) / \prefix{\pGen}(\units_{<t})$ where $\unit_t \in \eosalphabet$.
Because $\prefix{\pGen}(\units_{<t} \unit_t)$ is a sub-sum of $\prefix{\pGen}(\units_{<t})$, the numerator vanishes whenever the denominator does. I adopt the convention $0/0 \defeq 0$, under which the conditional is defined for every prefix $\units_{<t}$ and is identically zero on prefixes of probability zero.
For $\unit_t = \eos$, I adopt the convention $\prefix{\pGen}(\units_{<t}\,\eos) \defeq \pGen(\units_{<t})$---the probability that the string ends exactly at $\units_{<t}$---since the prefix probability is otherwise defined only for strings in $\kleene{\alphabet}$.
This conditional specifies the probability of each next unit $\unit_t \in \eosalphabet$ given the context---including the probability of ending the string by generating $\eos$.
Thus, every language model $\pGen$ decomposes as a product of these conditionals as
\begin{equation}\label{eq:autoregressive}
    \pGen(\units) = \prefix{\pGen}(\eos \mid \units) \prod_{t=1}^{\strlen} \prefix{\pGen}(\unit_t \mid \units_{<t}),
\end{equation}
where $\units = \unit_1 \cdots \unit_\strlen$ is a string of length $\strlen$.\footnote{Note that $\strlen$ is itself a random variable in this formulation.}
I call this a language model's \defn{autoregressive factorization}.
Furthermore, \Cref{eq:autoregressive} allows one to view $\pGen$ as a collection of conditional distributions $\{\prefix{\pGen}(\cdot \mid \units)\}_{\units \in \kleene{\alphabet}}$, each over $\eosalphabet$.

\paragraph{From conditionals to language models.}
The converse direction is more delicate.
A \defn{conditional collection} is any family $\{\prefix{\pGen}(\cdot \mid \units)\}_{\units \in \kleene{\alphabet}}$ of distributions over $\eosalphabet$, one for each prefix, where $\prefix{\pGen}$ is not assumed to arise from a language model.
A conditional collection can be used to compute the probability of a string by plugging the appropriate conditionals into \Cref{eq:autoregressive}.
However, the resulting measure over $\kleene{\alphabet}$ may fail to sum to $1$: probability mass can ``leak'' onto infinite sequences that never produce $\eos$.
A conditional collection for which no such leakage occurs, i.e., for which $\sum_{\units \in \kleene{\alphabet}} \pGen(\units) = 1$, is called \defn{tight}.
A characterization of tightness is given by \citet[Theorem~4.7]{du2023}.

\section{Technical Conditions for the Tautology}\label{sec:tautology-formal}

In \Cref{sec:intro}, I argued that any non-negative difficulty measure can be converted into a language model whose surprisal reproduces it; see \Cref{eq:tautology-construction}.
I now make this precise.
Given a difficulty measure $\difficulty \colon \eosalphabet \times \kleene{\alphabet} \to \reals_{\geq 0}$, the construction in \Cref{eq:tautology-construction} defines a conditional collection $\qdiff$ over $\eosalphabet$:
\begin{equation}\label{eq:construction}
\!\qdiffprefix(\unit_t \mid \units_{<t}) \defeq \frac{\exp({-\difficulty(\unit_t \mid \units_{<t})})}{\sum_{\unit' \in \eosalphabet} \exp({-\difficulty(\unit' \mid \units_{<t})})}.
\end{equation}
The technical subtlety is that $\qdiff$ as defined in \Cref{eq:construction} is not guaranteed to be a language model, as discussed in \Cref{sec:preliminaries}.
One could translate \citeposs{du2023} exact characterization of the tightness of conditional collections directly into the notation of $\difficulty$, but the resulting expression is cumbersome.
Instead, I give a simpler sufficient condition purely in terms of $\difficulty$, which admits a more lucid discussion about cognitive plausibility.\looseness=-1
\begin{restatable}{proposition}{proptautology}\label{prop:tautology}
Let $\difficulty \colon \eosalphabet \times \kleene{\alphabet} \to \reals_{\geq 0}$ be a non-negative difficulty measure, and let $\qdiff$ be defined by \Cref{eq:construction}.
If there exists a function $f \colon \mathbb{Z}_{>0} \to \reals_{\geq 0}$ such that
\begin{equation}\label{eq:sufficient}
    \difficulty(\eos \mid \units) \leq f(|\units| + 1) \quad \text{for all } \units \in \kleene{\alphabet}
\end{equation}
and $\sum_{t=1}^{\infty} \exp({-f(t)}) = \infty$,
then \begin{enumerate*}[label=(\alph*)]
    \item $\qdiff$ is tight (see \cref{sec:preliminaries}), i.e., $\sum_{\units \in \kleene{\alphabet}} \qdiff(\units) = 1$; and
    \item if moreover $f$ is constant---i.e., there exists $C > 0$ such that $\difficulty(\eos \mid \units) \leq C$ for all $\units \in \kleene{\alphabet}$---then $\qdiff$ has finite expected length: $\E_{\units \sim \qdiff}[|\units|] < \infty$.
\end{enumerate*}
\end{restatable}
\noindent The proof is deferred to \cref{app:proof-tautology}.

\paragraph{Cognitive plausibility.}
I now discuss the cognitive plausibility of part~(b) of \Cref{prop:tautology}.
In words, it requires that $\difficulty(\eos \mid \units)$---the difficulty of wrapping up the sentence after a prefix $\units$---be uniformly bounded.
Reading times at sentence-final words do exhibit a well-documented wrap-up cost \citep{smith2013,meister2022}, but this cost is a fixed penalty associated with sentence-boundary integration and does not appear to scale with the length of the preceding context.
The assumption that $\difficulty$ is bounded---which I adopt throughout---therefore satisfies the condition for any realistic difficulty measure.\footnote{Concretely, if $\difficulty(\unit \mid \units) \leq C$ for all $\unit$ and $\units$, then the constructed model satisfies $\qdiffprefix(\unit \mid \units) \geq \exp(-C)/|\eosalphabet|$, so $\qdiff$ itself meets the assumption in \Cref{def:metric}.}
Part~(b) is needed as a technical condition in \cref{sec:tautology}, where the scaling implication requires $\qdiff$ to have finite expected length.\looseness=-1

\section{The Scaling Implication}\label{sec:tautology}\label{sec:option1}\label{sec:option2}
The corpus assumption (\cref{sec:intro-propensity}) identifies the target language model with the data-generating distribution: $\pLM = \pData$.
A direct consequence is that one should choose as expressive a family $\family$ as possible, and, moreover, the richer $\family$ is, the better the estimator can approximate $\pData$, and the closer its surprisal values should be to the true surprisal under $\pData$.

\subsection{Measuring Psychometric Fitness}
To formalize the scaling implication, I require a measure of how well a language model's surprisal values predict processing difficulty.
In practice, researchers operationalize this fit by regressing behavioral measures (reading times, fixation durations, or ERP amplitudes) against surprisal values and reporting a goodness-of-fit statistic---typically delta log-likelihood \citep{goodkind2018,wilcox2023}, $R^2$ \citep{giulianelli2024}, or mean-squared error \citep{kiegeland2024}.
I abstract over these choices with a general definition.\looseness=-1

\begin{definition}[Fitness measure]\label{def:metric}
Let $\pGen$ be a language model with finite expected length (\cref{sec:preliminaries}) and let $\scoring \colon \reals_{\geq 0} \times \reals_{\geq 0} \to \reals$ be a bounded, continuous function.\footnote{\label{fn:eps-floor}I assume throughout that per-unit probabilities are uniformly bounded away from zero: there exists $\eta > 0$ such that $\prefix{\qLM}(\unit \mid \units) \geq \eta$ for every $\qLM \in \family$, $\unit \in \eosalphabet$, and $\units \in \kleene{\alphabet}$; necessarily $\eta \leq 1/|\eosalphabet|$, since the conditional probabilities sum to one.
Softmax-based models whose parameters range over a compact set satisfy this assumption.
Two consequences are used in what follows: first, surprisal values lie in the compact interval $[0, \log(1/\eta)]$; second, by the geometric argument in the proof of \cref{prop:tautology}, part~(b), every $\qLM \in \family$ is tight with $\E_{\units \sim \qLM}[|\units|] \leq (1-\eta)/\eta$.}
The $(\pGen, \scoring)$-\defn{fitness measure} is defined as\looseness=-1
\begin{align}\label{eq:metric}
    \metric(&\qLM) \defeq \\
    &\sum_{\unit \in \eosalphabet} \sum_{\units \in \kleene{\alphabet}} \prefix{\pGen}(\units\unit)\, \scoring\big(\difficulty(\unit \mid \units), -\log \prefix{\qLM}(\unit \mid \units)\big), \notag
\end{align}
where $\prefix{\pGen}(\units\unit)$ is the prefix probability of $\units\unit$ under $\pGen$, as defined in \cref{sec:background}, with the convention $\prefix{\pGen}(\units\,\eos) \defeq \pGen(\units)$.
\end{definition}
For delta log-likelihood, $R^2$, and mean-squared error, boundedness of $\scoring$ can be imposed without loss of generality due to the fact that the per-unit probabilities are bounded away from zero. 
See footnote~\ref{fn:eps-floor}: together with the boundedness of $\difficulty$ (\cref{sec:tautology-formal}), it confines the arguments of $\scoring$ to the compact rectangle $[0, C] \times [0, \log(1/\eta)]$ with $\eta > 0$, and $\metric$ depends on $\scoring$ only through its values there; any continuous $\scoring$ may therefore be truncated outside this rectangle without changing $\metric$, so the boundedness requirement sacrifices no generality.

\begin{restatable}[Continuity of $\metric$]{lemma}{lemcontinuity}\label{rem:continuity}
Suppose $\pGen$ has finite expected length and $\scoring$ is continuous and bounded.
Then $\metric \colon \family \to \reals$ is continuous.
\end{restatable}
\noindent The proof is deferred to \cref{app:proof-continuity}.

\subsection{Estimation}
Given a $(\pData, \scoring)$-fitness measure $\metric$, one seeks the language model $\qLM^* = \operatorname*{arg\,max}_{\qLM \in \family} \metric(\qLM)$ that best predicts processing difficulty.\footnote{As discussed in \cref{sec:intro-propensity}, \citet{kiegeland2024} pursued exactly this strategy; as they themselves note, it raises the question of what, if anything, the optimized model explains.}
The choice of $\family$ is critical.
If $\family$ is the set of all language models---all distributions over $\kleene{\alphabet}$---then by \Cref{prop:tautology}, $\qLM^*$ can perfectly accommodate any observed difficulty pattern, and surprisal theory has no grounding.
The smaller $\family$ is, the more constrained and falsifiable surprisal theory becomes.
Under the corpus assumption, however, one need not justify $\family$ on cognitive grounds at all: the goal is simply to approximate $\pData$ as well as possible, so in principle one would choose the most expressive family available.\looseness=-1

Because one does not have access to $\pData$, a standard method for selecting $\qLM \in \family$ under the corpus assumption is \defn{maximum-likelihood estimation} (MLE).\footnote{\label{fn:sgd}In practice, neural language models are trained by stochastic gradient descent (SGD) on the cross-entropy loss, which is the empirical analogue of the KL objective but does not guarantee convergence to the global MLE.
The loss landscape of transformer language models is non-convex with many local minima and saddle points, and SGD may converge to a spurious local optimum rather than the global KL-minimizer.
The scaling implication, as stated, is therefore a claim about the \emph{population-level} optimum within $\family$---the best model that \emph{could} be found---not about the output of any particular training run.
Additional theory, perhaps drawing on the empirical observation that overparameterized networks tend to find good minima at scale, may be necessary to bridge this gap.} Given a corpus $\units^{(1)}, \ldots, \units^{(N)} \overset{\text{i.i.d.}}{\sim} \pData$ of $N$ strings, the MLE is
\begin{equation}
    \qLM_N \defeq \operatorname*{arg\,max}_{\qLM \in \family} \frac{1}{N} \sum_{n=1}^{N} \log\qLM(\units^{(n)}).
\end{equation}
Under standard regularity conditions\emdash compactness of $\family$ and a uniform integrability condition on $\log \qLM$, both made precise in \Cref{thm:consistency} of \cref{sec:consistency}---the estimator converges almost surely to the KL projection of $\pData$ onto $\family$:
\begin{equation}\label{eq:kl-convergence}
    \KL{\pData}{\qLM_N} \xrightarrow[N \to \infty]{\text{a.s.}} \inf_{\qLM \in \family} \KL{\pData}{\qLM}.
\end{equation}
Let $\qLM^* \defeq \operatorname*{arg\,min}_{\qLM \in \family} \KL{\pData}{\qLM}$ denote the KL projection of $\pData$ onto $\family$.
When $\family$ is well-specified---i.e., $\pData \in \family$---the right-hand side of \Cref{eq:kl-convergence} is zero and $\qLM^* = \pData$.
Under misspecification, $\qLM^*$ is the member of $\family$ closest to $\pData$ in KL divergence.
In either case, the following holds.\footnote{When identifiability fails, the conclusion of \Cref{prop:monotone} still holds provided $\metric$ takes the same value on all elements of $\family^*$, since every limit point of $\{\qLM_N\}$ lies in $\family^*$ by \Cref{thm:consistency}.}

\begin{restatable}[Scaling implication]{proposition}{propmonotone}\label{prop:monotone}
Assume $\family^*$ is a singleton $\{\qLM^*\}$, i.e., the KL-minimizer is unique.
Under the regularity conditions of \Cref{thm:consistency}, it holds that
\begin{equation}\label{eq:monotone}
    \metric(\qLM_N) \xrightarrow[N \to \infty]{~p~} \metric(\qLM^*).
\end{equation}
In particular, for any $\epsilon > 0$ and $\delta > 0$, there exists $N_0$ such that for all $N \geq N_0$, it holds that
$\Pr\big(|\metric(\qLM_N) - \metric(\qLM^*)| < \epsilon\big) > 1 - \delta$.
\end{restatable}
\noindent The proof is deferred to \cref{app:proof-monotone}.

\Cref{prop:monotone} establishes that fitness \emph{converges} to $\metric(\qLM^*)$ as the estimator improves---but it does not predict that fitness increases monotonically with more data or larger models, which is the informal reading of ``scaling.''
Thus, the convergence result itself is not what is falsified.
What the corpus assumption additionally predicts---and what \citeposs{oh2023} findings contradict---is that $\pData$ \emph{maximizes} $\metric$, i.e., that the true corpus distribution is the best predictor of processing difficulty.
If $\metric$ first increases and then decreases as models approach $\pData$, then an imperfect approximation to $\pData$ is a better predictor than $\pData$ itself---a reversal that misspecification alone cannot produce.\footnote{Strictly, this inference requires that the largest models approach $\pData$ itself, not merely the KL projection $\qLM^* \neq \pData$: under misspecification, a rise-and-fall of $\metric$ along the path to $\qLM^*$ is compatible with $\pData$ maximizing $\metric$.
The expressiveness of modern architectures makes approximate well-specification plausible, but it is an assumption, not a theorem.}

Even before this empirical falsification, the corpus assumption deserved more scrutiny: $\pData$ is determined by who wrote the texts, when, and for what purpose---a model estimated from Wall Street Journal prose \citep{marcus1993} converges to a different distribution than one estimated from Reddit \citep[\emph{cf.}][]{radford2019}, encoding corpus-specific statistics (genre, topic, editorial style) with no clear role in accounts of syntactic processing.
The grounding problem is not solved but deferred: from ``which $\pLM$?''\ to ``which corpus?''\looseness=-1

\section{The Search for a Rationalist Grounding}
\label{sec:rationalist}
If one accepts the corpus assumption as falsified and the early rationalist commitments
of \citet{hale2001} abandoned with the move to neural models, the question
becomes whether surprisal theory can be rescued by choosing $\family$ on
cognitive grounds.
What I call the \defn{rationalist principle} is the commitment to grounding $\family$ and $\pLM$ in aspects of human cognition motivated independently of the behavioral data surprisal theory is meant to explain.\footnote{I use ``rationalist'' loosely.
Specifically, I was inspired by the nativist tradition of \citet{chomsky1966}.}
The clearest examples are constraints on the comprehender's processing architecture---working memory capacity, incremental processing---established through non-reading-time paradigms.

\subsection{Lossy-Context Surprisal}
\citet{futrell2020} propose replacing the full context $\units_{<t}$ with a lossy compression $\memory(\units_{<t})$, where $\memory$ is a deletion noise process reflecting memory decay.
If $\memory$ is derived from independently motivated assumptions about human memory rather than fit to reading-time data, this move can return propensity to surprisal theory.
\citet{hahn2022} scale this framework by parameterizing $\memory$ as a resource-rational retention policy optimized to minimize expected surprisal under a fixed resource budget---a rational analysis of the comprehender's processing problem, not a fit to behavioral data.
They confirm the model's predictions in three experiments on center-embedded relative clauses, producing fine-grained predictions that differ sharply from standard surprisal and that are borne out empirically.
However, the grounding is partial: $\memory$ is derived from first principles, but the underlying $\pLM$ (GPT-2) is not.
A complementary route is to build the memory limitation into the architecture itself: \citet{timkey2023} construct a recurrent language model with a single self-attention head whose limited retrieval capacity mirrors cue-based retrieval theories of working memory, and show that it captures semantic and syntactic interference effects observed in human sentence processing---thereby grounding the choice of $\family$ itself, not merely $\memory$, in independently motivated memory constraints.
\citet{mccurdy2024} and \citet{devarda2024} extend this line of work, showing that cognitively justified context degradation improves alignment with human reading times.
The promise of lossy-context surprisal lies in the independent derivation of $\memory$: a cognitively motivated memory function can break the tautology in a way that an empirically optimized one---which would merely shift the circularity from $\pLM$ to $\memory$---cannot.

\subsection{Developmentally Plausible Corpora}
\label{sec:small-data}
A different strategy is the BabyLM approach \citep{warstadt2023,hu2024,charpentier2025}, which restricts training data to developmentally plausible corpora, e.g., drawing on the CHILDES corpus \citep{macwhinney2000}, thereby grounding the corpus choice in language acquisition research rather than reading-time correlations.
Two pitfalls remain.
First, with a developmentally realistic amount of input, the estimator lacks an asymptotic interpretation, and the resulting model is dominated by its inductive bias (architecture, optimizer, initialization).
Unless that bias is independently justified on cognitive grounds, the grounding problem reappears.\footnote{One might object that regularization implicitly constrains $\family$ in cognitively relevant ways. If regularization is what makes intermediate-sized models predict reading times better, that would constitute evidence for a simplicity constraint on $\pHuman$---a promising direction, but one requiring independent theoretical motivation.\looseness=-1}
Second, this approach conflates ``what children hear'' with ``how adult humans process language.''
One may be able to resolve this conflation satisfactorily---for instance, by arguing that adult processing reflects the statistics of childhood input---but the argument requires independent support, of which I am unaware.

\section{Anticipated Objections}\label{sec:objections}
I close by addressing two objections that a committed proponent of surprisal theory might raise.

\paragraph{Bounded Rationality.}
First, one might respond that the corpus assumption was never unprincipled.
Under rational analysis \citep{anderson1990}, the comprehender's expectations should approximate the statistics of the environment because optimal prediction demands it, so the corpus assumption is itself derived from an independently motivated commitment.
One could argue that the degradation documented by \citet{oh2023} is thus evidence of \emph{bounded} rationality---not a collapse of grounding.
I welcome this rebuttal because it concedes the thesis.
Indeed, I contend that it is just a reframing of my argument.
It grounds $\pLM$ in a rationalist commitment, i.e., optimal prediction subject to resource constraints, that is external to the behavioral data.
Moreover, the appeal to bounded rationality forces one to enumerate those resource constraints.
If the bounds are left unspecified, any departure from optimal prediction can be attributed to \emph{some} bound after the fact, and the tautology reappears with the bound, rather than $\pLM$, as the free parameter.
Enumerating the constraints on independent grounds, e.g., the memory limitations that underlie lossy-context surprisal (\cref{sec:rationalist}), is precisely what the rationalist principle demands.

\paragraph{Auxiliary Hypotheses.}
Second, one might object, following \citet{lakatos1970}, that no theory is tested in isolation.
Instead, it is the case that the core commitments of a theory are examined jointly with a number of auxiliary hypotheses.
Newton's second law is a classic illustration of this point:
\begin{equation}\label{eq:second-law}
F = ma.
\end{equation}
Here, $F$ denotes a force acting on a body, $m$ the body's mass, and $a$ the resulting acceleration.
If one \emph{defines} a force as whatever $F$ makes the equation hold, then, taken on its own, \Cref{eq:second-law} is unfalsifiable.
\Cref{eq:second-law} acquires empirical content only once the forces are specified independently, e.g., as with universal gravitation or Hooke's law.
Such force laws are fixed in advance, drawn from a small family, and tested on phenomena to which they were not fitted; the language model is surprisal theory's counterpart of the force law.
Although Newtonian mechanics is thus open to a seemingly similar charge of tautology, no one concludes that it is tautological as a whole.
On \citeposs{lakatos1970} account, this is because a theory is appraised not by the falsifiability of any single prediction but by whether it is \emph{progressive}, i.e., whether its adjustments continue to predict novel facts about the world, or whether it is \emph{degenerating}, i.e., whether they merely accommodate anomalies after the fact.
Newtonian mechanics has repeatedly met this standard, e.g., the existence and location of Neptune were inferred from anomalies in the orbit of Uranus before the planet was identified \citep{grosser1962}.
Crucially, the same strategy could, and did, fail: a further planet, Vulcan, was posited to explain anomalies in the orbit of Mercury, but it was never found \citep{baum1997}.
On my view, recent practice in psycholinguistics, I contend, looks more like the search for Vulcan than the search for Neptune.
For instance, \citet{yoshida2026} demonstrate that surprisal theory's failure to predict garden-path effects \citep{huang2024} can be absorbed by fine-tuning the language model on the reading times themselves.
Indeed, re-estimating the language model in light of new data gives surprisal theory a clean mechanism to accommodate anomalies.
By \Cref{prop:tautology}, a suitable language model always exists.
Thus, re-estimation itself is a symptom of the problem, rather than a solution.

\section{Conclusion}
This paper has argued that surprisal theory, as currently practiced, is tautological.
I show that, for any pattern of processing difficulty, a language model can be constructed whose surprisal reproduces it (\Cref{prop:tautology}), so the theory makes no falsifiable prediction until the language model is constrained from outside.
Further, in an archaeology of the literature, I identify the corpus assumption, which has implicitly undergirded research in surprisal theory: that the relevant language model is the very distribution that generated the training corpus, so that better corpus fit should yield better predictions of processing difficulty (\Cref{prop:monotone}).
This scaling implication is what made surprisal theory falsifiable in practice, and, indeed, it is what recent evidence has falsified.
As \citet{oh2023} document, the fit between surprisal and reading times improves as language models grow larger only up to a point after which it falls off---even as the models become better approximations to $\pData$.
All in all, I contend this leaves surprisal theory---at present---without clean grounding, i.e., the language model under which surprisal is computed can neither be left as a free parameter nor be pinned down by any current principle in the literature.

\section*{Limitations}
This paper is a conceptual and mathematical argument about the logical status of surprisal theory; it does not present new experiments or computational results.
Several limitations follow from this scope.
First, the argument that the corpus assumption has been falsified rests on the findings of \citet{oh2023}, replicated by \citet{oh2024} across multiple model families and corpora, and extended to neural measures by \citet{lin2025}. While I treat these findings as robust, future work may refine or qualify them---for instance, by disentangling the effects of model architecture from those of corpus fit.
Second, the sufficient condition in \Cref{prop:tautology} is mild and likely satisfied by any realistic difficulty measure, but I have not verified this empirically for specific psychometric datasets.
Third, the paper identifies the grounding problem but does not solve it: I sketch two partial approaches (lossy-context surprisal and developmentally plausible corpora) without providing a complete, independently justified specification of $\pLM$.
Fourth, \Cref{thm:consistency} defines the estimator $\qLM_N$ as the $\operatorname{arg\,max}$ of the empirical log-likelihood over $\family$, and the scaling implication inherits this idealization: the guarantee concerns the model a perfect optimizer \emph{would} find.
Perfect optimization is likely unattainable in practice---even training a three-node neural network to fit a dataset is NP-complete \citep{blum1992}---and gradient-based training offers no guarantee of reaching the global optimum; bridging this gap would require optimization-aware theory beyond the scope of this paper.
Finally, the analogy to the propensity interpretation of fitness in evolutionary biology is meant to be structurally illuminating, not to claim that the two cases are identical in all respects.

\section*{Acknowledgements}
I thank Eleanor Chodroff, Francesco Ignazio Re, Andreas Opedal, Richard Futrell, Samuel Kiegeland, Taiga Someya, Timothy J. O'Donnell, Tim Vieira, Kate McCurdy, Thomas Hikaru Clark, Jacob Hoover Vigly, and Zach Hopton for helpful discussion and feedback on drafts of this paper.

\section*{AI Declaration}
Claude Code (Anthropic) was used for proofreading, phrasing, and checking the mathematical arguments.
All claims, arguments, and proofs are my own, and I take full responsibility for the content of the paper.
Any remaining errors belong to me.

\bibliography{custom}

\clearpage
\onecolumn
\appendix

\etocdepthtag.toc{appendix}
\etocsettagdepth{mainmatter}{none}
\etocsettagdepth{appendix}{subsection}
\etocsetnexttocdepth{subsection}
\etocsettocstyle
  {\section*{Contents of the Appendix}}
  {}
\tableofcontents
\clearpage

\section{Consistency of the MLE for Language Models}\label{sec:consistency}

This appendix gives a self-contained proof that the maximum-likelihood estimator converges to the KL projection of the data-generating distribution onto the model family.
A proof is included here, rather than merely cited, because the standard references \citep{white1982,white1989,vandervaart1998} do not explicitly verify the uniform domination condition for language models, i.e., distributions over $\kleene{\alphabet}$.
As I show below, this condition reduces to the assumption that $\pData$ has finite expected length.
The result below is stated entirely in function space, i.e., the family $\family$ is treated as a compact set of language models, with no reference to an underlying parameter space.
Following \citet{white1989}, the theorem shows convergence to the \emph{set} of KL-minimizers, so no identification assumption is required, which is essential for neural language models, whose parameter symmetries (e.g., permutation of attention heads) make point identification impossible.

An important caveat: the theorem assumes that the MLE $\qLM_N$ is a \emph{global} maximizer of the empirical log-likelihood over $\family$.
In practice, neural language models are trained by stochastic gradient descent on a non-convex loss landscape, and the optimizer may converge to a local rather than global optimum (see footnote~\ref{fn:sgd} in the main text).
Additionally, they are often not trained to converged at all.
The scaling implication, and hence the consistency result, is therefore a statement about the population-level optimum within $\family$---the best model that \emph{could} be found---not about the output of any particular training run.

\begin{theorem}[Consistency of the MLE]\label{thm:consistency}
Let $\alphabet$ be an alphabet with $|\alphabet| = V$, and let $\pData$ be a data-generating distribution over $\kleene{\alphabet}$.
Next, let $\family$ be a family of language models over $\kleene{\alphabet}$, equipped with the topology of pointwise convergence.\footnote{\label{fn:pointwise}Formally, $\qLM_n \to \qLM$ in $\family$ if and only if $\qLM_n(\units) \to \qLM(\units)$ for every $\units \in \kleene{\alphabet}$.
Because $\kleene{\alphabet}$ is countable, the topology of pointwise convergence is metrizable, so sequential continuity and continuity coincide on $\family$.}
Additionally, let $\family^* \defeq \operatorname{arg\,min}_{\qLM \in \family} \KL{\pData}{\qLM}$ be the set of KL-minimizers, let $Q(\qLM) \defeq \E_{\units \sim \pData}[\log(\qLM(\units))]$ denote the expected log-likelihood, and let $\qLM_N$ denote the maximum-likelihood estimator of $\pData$ within $\family$ computed from an i.i.d.\ corpus $\units^{(1)}, \ldots, \units^{(N)} \sim \pData$.
If the following three conditions hold:
\begin{enumerate}[label=(\roman*)]
    \item \textbf{Compactness.} $\family$ is compact.
    \item \textbf{Uniform domination.} There exists a $\pData$-integrable function $d \colon \kleene{\alphabet} \to \reals_{\geq 0}$ such that $\sup_{\qLM \in \family} |\log(\qLM(\units))| \leq d(\units)$ for all $\units \in \kleene{\alphabet}$.
    \item \textbf{Finite expected length.} $\E_{\units \sim \pData}[|\units|] < \infty$.
\end{enumerate}
Then, almost surely, every limit point of the MLE sequence $\{\qLM_N\}_{N=1}^\infty$ lies in $\family^*$.
In particular,
\begin{equation}
    \KL{\pData}{\qLM_N} \xrightarrow[N \to \infty]{\text{a.s.}} \inf_{\qLM \in \family} \KL{\pData}{\qLM}.
\end{equation}
If $\family^*$ is a singleton $\{\qLM^*\}$, then $\qLM_N \to \qLM^*$ pointwise a.s.
\end{theorem}

\paragraph{Discussion of the assumptions.}
For parametric families---such as neural language models with bounded weights---compactness of the parameter space and continuity of the parameterization give compactness of $\family$ in the pointwise topology, provided all models in $\family$ are tight.
Showing tightness, however, is non-trivial: even with a compact parameter space, a sequence of language models can converge pointwise to a non-tight model if the expected length diverges along the sequence.
Because $[0,1]^{\kleene{\alphabet}}$ is compact in the pointwise topology by Tychonoff's theorem, $\family$ is compact if and only if it is closed in $[0,1]^{\kleene{\alphabet}}$, i.e., if and only if every pointwise limit of members of $\family$ is again a member of $\family$ and, in particular, a language model.
Note that finiteness of each model's expected length does not suffice---expected lengths may diverge along a sequence.
A \emph{uniform} bound, by contrast, does suffice.
Suppose $\E_{\units \sim \qLM}[|\units|] \leq L$ for every $\qLM \in \family$, and let $\qLM_n \to \qLM_\infty$ pointwise, where a priori $\qLM_\infty$ is only a sub-probability measure on $\kleene{\alphabet}$.\footnote{That $\qLM_\infty$ has total mass at most $1$ follows from Fatou's lemma applied to counting measure on $\kleene{\alphabet}$.
In particular, $\sum_{\units \in \kleene{\alphabet}} \qLM_\infty(\units) = \sum_{\units \in \kleene{\alphabet}} \lim_{n \to \infty} \qLM_n(\units) \leq \liminf_{n \to \infty} \sum_{\units \in \kleene{\alphabet}} \qLM_n(\units) = 1$.}
For any $T \in \mathbb{Z}_{>0}$, only finitely many $\units \in \kleene{\alphabet}$ satisfy $|\units| \leq T$, so pointwise convergence gives
\begin{equation}
\!\sum_{\units \in \kleene{\alphabet}} \ind{|\units| \leq T}\, \qLM_\infty(\units)
    = \lim_{n \to \infty} \sum_{\units \in \kleene{\alphabet}} \ind{|\units| \leq T}\, \qLM_n(\units)
    = \lim_{n \to \infty} \Big(1 - \Pr_{\units \sim \qLM_n}\big(|\units| > T\big)\Big)
    \geq 1 - \frac{L}{T},
\end{equation}
where the inequality is Markov's.
Letting $T \to \infty$ yields $\sum_{\units \in \kleene{\alphabet}} \qLM_\infty(\units) = 1$, i.e., $\qLM_\infty$ is a language model.
For~(ii), if every $\qLM \in \family$ assigns each conditional probability at least $\eta > 0$ (the standing assumption of footnote~\ref{fn:eps-floor} in \Cref{def:metric}, which softmax-based models on a compact parameter space satisfy), then $|\log(\qLM(\units))| \leq (|\units|+1) \log(1/\eta)$, and the domination condition reduces to (iii).
The same assumption yields the uniform expected-length bound $\E_{\units \sim \qLM}[|\units|] \leq (1-\eta)/\eta$ across $\family$ (by the geometric argument in the proof of \Cref{prop:tautology}), which, by the argument just given, ensures that pointwise limits of members of $\family$ are language models.
The finite expected length assumption~(iii) plays a dual role: it ensures the domination condition~(ii) for the consistency theorem and the well-definedness of the fitness measure $\metric$ (\Cref{def:metric}), since both require that sums over prefixes weighted by $\pDataprefix$ converge.
This is the key assumption not accounted for by the standard references \citep{white1982,vandervaart1998}, which work in finite-dimensional parameter spaces and do not consider the structure of distributions over $\kleene{\alphabet}$.

\begin{proof}
Define the empirical log-likelihood $Q_N(\qLM) \defeq \frac{1}{N}\sum_{n=1}^N \log(\qLM(\units^{(n)}))$, so that $\qLM_N = \operatorname{arg\,max}_{\qLM \in \family}\, Q_N(\qLM)$.
Note that $Q(\qLM) = -\KL{\pData}{\qLM} - \entropy(\pData)$, where $\entropy(\pData)$ is finite because finite expected length (assumption~(iii)) implies $\entropy(\pData) \leq (\E_{\units \sim \pData}[|\units|] + 1)\log|\eosalphabet| < \infty$; hence maximizing $Q$ is equivalent to minimizing KL divergence.
The argument follows \citet{white1982,white1989} and proceeds in four steps.

\begin{itemize}[leftmargin=*]
\item \textbf{Step 1: The expected log-likelihood $Q$ is finite.}
For each $\qLM \in \family$, the integrand $|\log(\qLM(\units))|$ is bounded by the dominating function $d(\units)$ from assumption~(ii).
Since $d$ is $\pData$-integrable by assumption, $\E_{\units \sim \pData}[|\log(\qLM(\units))|] \leq \E_{\units \sim \pData}[d(\units)] < \infty$.
Therefore $Q(\qLM) = \E_{\units \sim \pData}[\log(\qLM(\units))]$ is finite for every $\qLM \in \family$.

\item \textbf{Step 2: $Q$ is continuous on $\family$.}
Let $\{\qLM_n\}_{n=1}^\infty$ be a sequence in $\family$ with $\qLM_n \to \qLM$ pointwise.
Then for each $\units \in \kleene{\alphabet}$, $\log(\qLM_n(\units)) \to \log(\qLM(\units))$ by continuity of the logarithm.
By assumption~(ii), $|\log(\qLM_n(\units))| \leq d(\units)$ for all $n$ and all $\units$, where $d$ is $\pData$-integrable.
The dominated convergence theorem \citep{billingsley2012} then gives
$Q(\qLM_n) = \E_{\units \sim \pData}[\log(\qLM_n(\units))] \to \E_{\units \sim \pData}[\log(\qLM(\units))] = Q(\qLM)$,
so $Q$ is sequentially continuous, and hence continuous, on $\family$, because the topology of
pointwise convergence on $\family$ is metrizable (footnote~\ref{fn:pointwise}).

\item \textbf{Step 3: Uniform law of large numbers.}
The strong law of large numbers gives $Q_N(\qLM) \to Q(\qLM)$ a.s.\ for each fixed $\qLM \in \family$.
I upgrade this to \emph{uniform} convergence over $\family$.
The family of functions $\{\units \mapsto \log(\qLM(\units)) : \qLM \in \family\}$ is dominated by the $\pData$-integrable function $d$ (assumption~(ii)).
By assumption~(i), $\family$ is compact, and it is metrizable (footnote~\ref{fn:pointwise}); moreover, the map $\qLM \mapsto \log(\qLM(\units))$ is continuous for each fixed $\units$, as established in the proof of Step~2.
The family is therefore a $\pData$-Glivenko--Cantelli class \citep[Example~19.8]{vandervaart1998}, which yields
\begin{equation}
    \sup_{\qLM \in \family} |Q_N(\qLM) - Q(\qLM)| \xrightarrow[N \to \infty]{\text{a.s.}} 0.
\end{equation}

\item \textbf{Step 4: Convergence of maximizers.}
The maximizer $\qLM_N$ exists because $Q_N$ is a finite sum of continuous functions on the compact set $\family$.
The uniform convergence of Step~3 holds on an event of probability one.
Fix a realization of this event $\units^{(1)}, \units^{(2)}, \ldots$.
Then, let $\qLM^{\dagger}$ be any limit point of the sequence $\{\qLM_N\}_{N=1}^\infty$.
By compactness~(i), at least one such limit point exists.
Passing to a subsequence $\{\qLM_{N_k}\}$ with $\qLM_{N_k} \to \qLM^{\dagger}$, the uniform convergence from Step~3 gives
\begin{equation}
    |Q_{N_k}(\qLM_{N_k}) - Q(\qLM^{\dagger})| \leq \underbrace{|Q_{N_k}(\qLM_{N_k}) - Q(\qLM_{N_k})|}_{\to\, 0 \text{ by Step 3}} + \underbrace{|Q(\qLM_{N_k}) - Q(\qLM^{\dagger})|}_{\to\, 0 \text{ by Step 2}} \to 0.
\end{equation}
Now, since $\qLM_{N_k}$ maximizes $Q_{N_k}$ over $\family$, for any $\qLM \in \family$, it holds that $Q_{N_k}(\qLM_{N_k}) \geq Q_{N_k}(\qLM)$.
Taking $k \to \infty$ and applying the uniform convergence of Step~3 to both sides, $Q(\qLM^{\dagger}) \geq Q(\qLM)$ for every $\qLM \in \family$.
Hence $\qLM^{\dagger} \in \operatorname{arg\,max}_{\qLM \in \family} Q(\qLM) = \family^*$.
Since $Q(\qLM) = -\KL{\pData}{\qLM} + \mathrm{const}$, this is equivalent to $\qLM^{\dagger} \in \operatorname{arg\,min}_{\qLM \in \family} \KL{\pData}{\qLM}$.
Because every limit point of $\{\qLM_N\}$ lies in $\family^*$, compactness of $\family$ and continuity of $Q$ (Step~2) imply that the KL divergence converges: $\KL{\pData}{\qLM_N} \to \inf_{\qLM \in \family} \KL{\pData}{\qLM}$ a.s.
If $\family^* = \{\qLM^*\}$ is a singleton, then every limit point equals $\qLM^*$, and by compactness the full sequence converges: $\qLM_N \to \qLM^*$ pointwise a.s.
\end{itemize}
\end{proof}

\section{Deferred Proofs}\label{app:proofs}

\subsection{Proof of \texorpdfstring{\Cref{prop:tautology}}{Proposition \getrefnumber{prop:tautology}}}\label{app:proof-tautology}

\proptautology*

\begin{proof}
\textbf{Part (a).} Each factor in \Cref{eq:construction} defines a valid conditional distribution over $\eosalphabet$ (by the softmax normalization), so $\qdiff$ is a conditional collection in the sense of \cref{sec:preliminaries}.
I apply \citet[Proposition~4.3]{du2023}, which states that a conditional collection is tight if there exists a function $h \colon \mathbb{Z}_{>0} \to \reals_{\geq 0}$ satisfying (i)~$\qdiffprefix(\eos \mid \units) \geq h(|\units| + 1)$ for all $\units \in \kleene{\alphabet}$, and (ii)~$\sum_{t=1}^{\infty} h(t) = \infty$.
I construct such an $h$.
Since $\difficulty \geq 0$, each term in the denominator of \Cref{eq:construction} satisfies $\exp(-\difficulty(\unit' \mid \units)) \leq 1$, so the denominator is at most $|\eosalphabet|$.
The numerator for $\eos$ is $\exp(-\difficulty(\eos \mid \units))$.
Therefore:
\begin{subequations}
\begin{align}
    \qdiffprefix(\eos \mid \units) &= \frac{\exp(-\difficulty(\eos \mid \units))}{\sum_{\unit' \in \eosalphabet} \exp(-\difficulty(\unit' \mid \units))} \\
    &\geq \frac{\exp(-\difficulty(\eos \mid \units))}{|\eosalphabet|}. \label{eq:eos-lower-bound}
\end{align}
\end{subequations}
By hypothesis, $\difficulty(\eos \mid \units) \leq f(|\units| + 1)$, so $\exp(-\difficulty(\eos \mid \units)) \geq \exp(-f(|\units| + 1))$.
Setting $h(t) \defeq \exp(-f(t)) / |\eosalphabet|$, condition~(i) is satisfied.
For condition~(ii): $\sum_{t=1}^{\infty} h(t) = \frac{1}{|\eosalphabet|} \sum_{t=1}^{\infty} \exp(-f(t)) = \infty$ by hypothesis.
Hence, $\qdiff$ is tight.

\textbf{Part (b).}
If $f$ is the constant function $f(t) = C$, then $\sum_{t=1}^{\infty} \exp(-f(t)) = \infty$, so $\qdiff$ is tight by part~(a).
Moreover, \Cref{eq:eos-lower-bound} gives $\qdiffprefix(\eos \mid \units) \geq \exp(-C) / |\eosalphabet|$ for all $\units$.
Setting $\rho \defeq 1 - \exp(-C)/|\eosalphabet| < 1$, the probability that a string sampled from $\qdiff$ has length at least $t$ is at most $\rho^t$, since such a string must fail to generate $\eos$ at each of its first $t$ positions.
Therefore, $\E_{\units \sim \qdiff}[|\units|] = \sum_{t=1}^{\infty} \Pr(|\units| \geq t) \leq \sum_{t=1}^{\infty} \rho^t = \rho/(1-\rho) < \infty$.
\end{proof}

\subsection{Proof of \texorpdfstring{\Cref{rem:continuity}}{Lemma \getrefnumber{rem:continuity}}}\label{app:proof-continuity}

\lemcontinuity*

\begin{proof}
Suppose $\qLM_n \to \qLM$ pointwise, i.e., $\qLM_n(\units) \to \qLM(\units)$ for all $\units \in \kleene{\alphabet}$.
Because every member of $\family$ is a language model---a probability distribution over $\kleene{\alphabet}$ with total mass $1$---Scheff\'e's lemma \citep{scheffe1947} upgrades pointwise convergence to convergence in $L^1$, i.e., $\sum_{\units \in \kleene{\alphabet}} |\qLM_n(\units) - \qLM(\units)| \to 0$.
Each prefix probability is a sum of string probabilities, so $\prefix{\qLM}_n(\units) \to \prefix{\qLM}(\units)$ for every $\units \in \kleene{\alphabet}$; and, because per-unit probabilities are bounded away from zero (\Cref{def:metric}), the limiting denominators satisfy $\prefix{\qLM}(\units) > 0$.
Hence the conditionals converge: $\prefix{\qLM}_n(\unit \mid \units) \to \prefix{\qLM}(\unit \mid \units)$ for all $\unit \in \eosalphabet$ and $\units \in \kleene{\alphabet}$.
Because $\scoring$ is continuous, the integrand converges pointwise for each fixed $\unit \in \eosalphabet$ and $\units \in \kleene{\alphabet}$.
Let $B$ be a bound on $|\scoring|$, which exists by assumption.
Because the conditionals sum to one, $\sum_{\unit \in \eosalphabet} \prefix{\pGen}(\units\unit) = \prefix{\pGen}(\units)$, and, because $\pGen$ has finite expected length, $\sum_{\units \in \kleene{\alphabet}} \prefix{\pGen}(\units) = \E_{\units \sim \pGen}[|\units|] + 1 < \infty$ \citep[Proposition~1]{opedal2024}; hence the weights $\prefix{\pGen}(\units\unit)$ are summable and $B \cdot \prefix{\pGen}(\units\unit)$ is an integrable dominating function.
Thus, for each $\unit$ and $\units$, the integrand $\prefix{\pGen}(\units\unit)\, \scoring\big(\difficulty(\unit \mid \units),\, -\log(\prefix{\qLM}_n(\unit \mid \units))\big)$ converges to $\prefix{\pGen}(\units\unit)\, \scoring\big(\difficulty(\unit \mid \units),\, -\log(\prefix{\qLM}(\unit \mid \units))\big)$, and all terms are bounded in absolute value by $B \cdot \prefix{\pGen}(\units\unit)$, which is summable.
The dominated convergence theorem \citep{billingsley2012} then permits interchange of the limit and the sum, yielding $\metric(\qLM_n) \to \metric(\qLM)$.
Because the pointwise topology on $\family$ is metrizable (footnote~\ref{fn:pointwise} in \cref{sec:consistency}), sequential continuity is equivalent to continuity, which completes the proof.\looseness=-1
\end{proof}

\subsection{Proof of \texorpdfstring{\Cref{prop:monotone}}{Proposition \getrefnumber{prop:monotone}}}\label{app:proof-monotone}

\propmonotone*

\begin{proof}
Because $\family^*$ is a singleton, \Cref{thm:consistency} gives $\qLM_N \xrightarrow[N \to \infty]{\text{a.s.}} \qLM^*$ pointwise.
By \Cref{rem:continuity}, applied with $\pGen = \pData$, which has finite expected length by assumption~(iii) of \Cref{thm:consistency}, $\metric$ is continuous on $\family$.
Hence, $\metric(\qLM_N) \xrightarrow[N \to \infty]{\text{a.s.}} \metric(\qLM^*)$ by the continuous mapping theorem \citep[\emph{cf.}][Theorem~2.3]{vandervaart1998}, and almost-sure convergence implies convergence in probability.\looseness=-1
\end{proof}

\end{document}